\documentclass[11pt]{article} 
\usepackage{rldmsubmit,palatino}
\usepackage{microtype}
\usepackage{booktabs}
\usepackage{hyperref}
\usepackage{color}

\usepackage{graphicx}
\usepackage{caption}
\usepackage{subcaption}

\usepackage{amsmath,amsfonts,amssymb,amsthm}
\usepackage{float}
\usepackage[mathscr]{euscript}
\usepackage[ruled]{algorithm}
\usepackage[noend]{algorithmic}
\usepackage{makecell}
\usepackage{tikz}
\usetikzlibrary{fit,calc}
\usepackage{wrapfig}
\usepackage{mathtools,bm}
\usepackage{siunitx}
\usepackage{stmaryrd}
\emergencystretch=\maxdimen
\newcommand{\fbf}{\mathbf{f}}

\newcommand{\hbf}{\mathbf{h}}

\newcommand{\obf}{\mathbf{o}}

\newcommand{\Ebf}{\mathbf{E}}

\newcommand{\Bcal}{\mathcal{B}}

\newcommand{\Dcal}{\mathcal{D}}

\newcommand{\Lcal}{\mathcal{L}}
\newcommand{\Mcal}{\mathcal{M}}

\newcommand{\Mscr}{\mathscr{M}}

\newcommand{\Pscr}{\mathscr{P}}

\newcommand{\Vscr}{\mathscr{V}}

\newcommand{\etabf}{\boldsymbol{\eta}}
\newcommand{\thetabf}{\boldsymbol{\theta}}

\newcommand{\phibf}{\boldsymbol{\phi}}

\newtheorem{theorem}{Theorem}

\title{Should Models Be Accurate?} 

\author{
Esra'a Saleh \\
Department of Computing Science\\
University of Alberta / Amii\\
Edmonton, AB, Canada \\
\texttt{esraa1@ualberta.ca} \\
\And
John D.~Martin \\
Department of Computing Science\\
University of Alberta / Amii\\
Edmonton, AB, Canada \\
\texttt{jmartin8@ualberta.ca} \\
\And
Anna Koop \\
Department of Computing Science\\
University of Alberta / Amii\\
Edmonton, AB, Canada \\
\texttt{akoop@ualberta.ca} \\
\And
Arash Pourzarabi\thanks{Arash Pourzarabi was killed when Flight PS752 was shot down by an Iranian surface-to-air missile on January 8, 2020.  His unfinished research on training generative models to produce experience for model-based RL was a signficant influence on the direction taken in this work.} \\
Department of Computing Science\\
University of Alberta / Amii\\
Edmonton, AB, Canada \\
\texttt{pourzara@ualberta.ca} \\
\And
Michael Bowling \\
Department of Computing Science\\
University of Alberta / Amii\\
Edmonton, AB, Canada \\
\texttt{mbowling@ualberta.ca} \\
}

\begin{document}
\maketitle

\begin{abstract}
Model-based Reinforcement Learning (MBRL) holds promise for data-efficiency by planning with model-generated experience in addition to learning with experience from the environment. However, in complex or changing environments, models in MBRL will inevitably be imperfect, and their detrimental effects on learning can be difficult to mitigate. In this work, we question whether the objective of these models should be the accurate simulation of environment dynamics at all. We focus our investigations on Dyna-style planning in a prediction setting. First, we highlight and support three motivating points: a perfectly accurate model of environment dynamics is not practically achievable, is not necessary, and is not always the most useful anyways. Second, we introduce a meta-learning algorithm for training models with a focus on their usefulness to the learner instead of their accuracy in modelling the environment.
Our experiments show that in a simple non-stationary environment, our algorithm enables faster learning than even using an accurate model built with domain-specific knowledge of the non-stationarity. 
\end{abstract}


\acknowledgements{We are grateful for the generous support in funding this work. This work was funded by the Natural Sciences and Engineering Research Council, the Alberta Machine Intelligence Institute, and the  Canadian
Institute For Advanced Research.  Computational resources were generously provided by Compute Canada.}  

\startmain 
\section{Introduction}
A promising approach to sample efficiency in reinforcement learning is model-based RL, where an agent learns a model of the environment.  The agent uses this model to generate its own \emph{internal experience} with which it can plan, reducing the need for more expensive \emph{veridical experience}, i.e., environment interactions. Learning effective models for planning remains a challenge in complex or changing environments \cite{nosuccessful_planning}. Models for planning are traditionally trained to accurately simulate environment dynamics. 
In complex or changing environments, models will necessarily be imperfect, which can lead to detrimental effects on the learner.

Past work has used several strategies to grapple with imperfect models. The first is to compensate in the planning process. An imperfect model can be used selectively by estimating predictive uncertainty to avoid planning in states where it could be detrimental \cite{abbas2020selective}. 
Similarly, while longer sequences of internal experience can often be beneficial in planning, their length needs to be limited due to increasing compounding model error \cite{holland2018effect}.
Another strategy is to relax the requirement for models to accurately simulate environment transitions by focusing only on accuracy in the resulting returns. This has been formally articulated through the value equivalence principle \cite{grimm2020value, farahmand2017value} and manifested in MuZero \cite{schrittwieser2020mastering}.

Our work questions the need for focusing on models that are concerned with accurately simulating the environment. We specifically study Dyna-style planning \cite{sutton1991dyna} in a prediction setting to provide a building block towards the control setting. We design a simple non-stationary environment that cannot be modelled perfectly without assumptions on how it changes. Then, we use meta-learning to train a model that generates synthetic transitions for fast adaptive learning.
The resulting learning speed upon change in the environment makes it competitive even with a stable accurate model of environment dynamics for an agent with privileged access to regions of stability. 



\section{Should Models Be Accurate?}
An important way RL systems can learn about their environment is by anticipating the outcomes of different behaviors using both veridical and internal experience. 
In the prediction setting, an RL system experiences a stream of observation vectors $\obf_t\in\mathbb{R}^d$ and scalar rewards $r_t$.
From this single stream of data, $\hbf_t\triangleq \obf_0,r_1,\obf_2,r_2,\obf_3,\cdots r_{t-1},\obf_t$, the learner forms an approximate value function (i.e., a prediction) to estimate the expected sum of future discounted rewards.
With linear function approximation, predictions are made by projecting a feature vector $\phibf_t = \fbf(\hbf_t)\in\mathbb{R}^n$ with weights $\thetabf$: 
\begin{align*}
    \hat{v}(\phibf_t; \thetabf_t) &\triangleq \thetabf_t^\top \phibf_t,&     \hat{v}(\phibf_t; \thetabf) &\approx v(\hbf_t) \triangleq \Ebf[R_{t}+\gamma R_{t+1} +\gamma^2R_{t+2}+\cdots |H_t = \hbf_t].
\end{align*}

Model-based learners update their prediction weights $\thetabf$ by planning with internal experience produced by a model $m \in\Mscr \subseteq \Pscr(\mathbb{R}^{n}\times\mathbb{R}\times\mathbb{R}^{n})$.
In this work, models combine the selection of initial states\footnote{In Dyna-style planning the selection of an initial state (and action) for planning is sometimes called search control~\cite{sutton1991dyna}.} with the dynamics of a forward transition.
At each planning step, the model generates an internal starting state $\tilde\phibf$, a next state  $\tilde\phibf'$, and a reward $\tilde r$.
Given this transition sample $(\tilde \phibf, \tilde r, \tilde \phibf') \sim m$, an update rule can be applied to the prediction weights $\thetabf$.

Dyna-style learning systems interleave planning steps with model-free updates on \textit{veridical experience} $(\phibf, r, \phibf')$, which come from the environment's dynamics $m^* \in\Mscr$.
In our work, Dyna learners apply semi-gradient updates to reduce one-step temporal difference (TD) errors $\delta \triangleq r + \gamma \thetabf^\top  \phibf' - \thetabf^\top \phibf$, using a fixed step size $\alpha\in\mathbb{R}^+$ at each time step: $\thetabf \gets \thetabf + \alpha \delta \phibf$.  This update is used with both veridical and internal experience.

In addition to learning predictions, model-based systems must also learn the planning model $m$.
These systems are often designed with an \textit{accuracy} objective---to approximate the environment dynamics so that eventually $m \approx m^*$.
Using an accurate model, the learner stands to benefit from additional transition samples that simulate the veridical experience.
Still, this raises an important question: should a learner strive for its model to be accurate?  
While accurate models might represent a subset of models that can be useful for learning, they do not represent the entirety of the set. Insisting on model accuracy, rather than model utility, can be a hindrance to fast adaptive learning.  We expound on this with three observations.

First, {\bf perfect accuracy in a model is not practically achievable}. Although accuracy is a well-defined objective for model learning, it represents a condition that cannot be achieved.
Accuracy is only achieved when the planning model and environment agree---specifically when $m=m^*$. 
In realistic environments, a learning system can never achieve this condition, because the real world is too big and complex (including other learning agents such as itself) to be discovered and represented with a finite amount of compute.
Therefore, fully-accurate models can never be learned.

Second, {\bf even when possible, an accurate model is not necessary}. In spite of the barrier to perfect accuracy, prior work shows that other models can work just as well.
For instance, the class of value-equivalent models are those which produce all the same value updates as $m^*$, but may transition under different dynamics \cite{grimm2020value}. 
Since they predict the same values, planning with value-equivalent models will lead to the same solution as an accurate model.
Systems such as Predictron \cite{silver2017predictron} and MuZero \cite{schrittwieser2020mastering} are successful examples which also show the benefit of connecting the model learning process to the way it is used for planning and value improvement. 
More importantly, this line of work shows that even if accuracy could be achieved, fully-accurate models are not necessary.

Third, {\bf an accurate model might not be the most useful}. 
In this work, planning models are used to approximate the true value function $v\in\Vscr$ by reducing squared TD error  with semi-gradient updates.
Solutions to this objective come from the set of TD fixed points $\boldsymbol \vartheta \subset \Vscr$.
Fixed points can be indexed with the model whose experience produced them: $\thetabf^{(m)} \in \Vscr$.

Given that fully-accurate models cannot be achieved, nor are they needed to reach a fixed point, one may wonder whether there exists models that are \textit{more useful} than others, or even more useful than the environment model!  
A natural measure of a model's utility could be the number of calls required to approach a TD fixed point.
With this, consider a subset of models $\Mcal\subset\Mscr$ with the same TD fixed point as the environment model $m^*$: $\Mcal \triangleq\{m \in \Mscr | \thetabf^{(m)} \in \boldsymbol \vartheta\}.$
The following result shows that there can be models that find solutions faster than the environment model when starting from the same point.

\begin{theorem}
There are domains where, for some initial $\thetabf\in\Vscr$, a TD fixed point $\thetabf^*\in\boldsymbol{\vartheta}$ can be approached with fewer calls to $m\in\Mcal$ than with the domain dynamics $m^*\neq m$. 
\end{theorem}
\begin{proof}[Proof Sketch]
Consider a domain that contains $k$ states, which share the same outgoing transitions and value.  Suppose the agent is employing a linear function approximator with states represented with a one-hot encoding.  Consider a model that produces internal experience from the $k$-hot vector containing these states, thus treating them as a single state.
Clearly $m\neq m^*$, and existing results from the state aggregation literature show such aggregations exist and do not change the TD fixed point.
Therefore, $m\in\Mcal$. However, it would take $k$ calls to $m^*$ to produce the update produced by one call to $m$.
\end{proof}

In short, a perfectly-accurate model is not practically achievable depending on the complexity of an environment, is not necessary by the value equivalence principle, and it might not be the most useful for a learner's objectives. A useful model in the prediction setting is one that helps the learning system obtain an accurate value function.


%
\section{A Demonstration of Useful Models}

In this section, we demonstrate how useful models can be trained to aid fast learning.  We focus on a prediction problem in a non-stationary environment where Dyna-style planning is used to rapidly adapt to changes in the environment.  The environment is constructed so that, despite being simple, an agent's state representation is not sufficient to build an accurate model.  We first describe the environment. Then, we present baselines that represent both model-free RL and model-based RL with models aimed at accuracy. After that, we introduce our algorithm that uses meta-learning to train a model aimed at usefulness. 
We seek to answer the question: \emph{can an algorithm that learns a useful, but not necessarily accurate model, result in faster learning than either an accurate model or an equivalent model-free learning approach?}

\paragraph{Non-Stationary Windy Hallway.}
To empirically compare the learning performance of planning models, we use a non-stationary windy hallway: a gridworld of 6 columns and 3 rows. The last column contains 3 episode-terminating grid cells. The starting location is in row 1 and column 0. In every timestep, an agent either moves one grid cell North-East or South-East, each with a probability of 0.5 (or possibly just East if already in a northern or southern-most cell). Every 300 episodes, the environment switches between two different reward regimes, where the location of non-zero reward changes. In one regime, reward is +1 for terminating in row 0, and 0 otherwise. In the other, reward is -1 for terminating in row 2 and 0 otherwise.
All algorithms use a discount factor of $\gamma=0.9$, and planning algorithms perform 5 planning updates per environment step. Experiments run for 75,000 steps (15,000 episodes) and each algorithm's value-function estimates are compared against the true expected return.

\paragraph{Baselines.}
Based on the environment, we select baselines to highlight learning speed when no planning is done and when Dyna-style planning is done with models aiming for accuracy. 
For our representative algorithm without planning, we use TD(0), which employs value updates from veridical experience only.  We call this algorithm {\it Model-free}.
To keep things simple, our baseline representatives of algorithms with Dyna-style planning do not explicitly train a model at all, but replay back veridical experience as the extreme of optimizing for accuracy.  We examine two alternatives for how past experience transition tuples are sampled.  The simplest option is to store every experienced transition tuple. A model will simply sample transition tuples proportional to how often they were experienced in the past.  We call the Dyna algorithm using this model {\it AllExperienceDyna}.  If the environment were Markov in the agent's state representation, this model would approach a perfectly accurate model.  As our environment is not Markov, this approach will result in obvious problems as internal experience from different reward regimes can distract it from the true expected return.
Our second baseline uses domain-specific knowledge of the environment to only store and resample veridical transitions tuples that are Markov in the agent's state representation (i.e., the non-terminating state transitions).  This experience is stable across changes in the reward regime and we call the Dyna algorithm with this model {\it StableExperienceDyna}.  This approach, in many ways, is an unfair comparison, as it is exploiting domain knowledge to construct a stable model.  We could, though, see this as an idealized representative of Abbas and colleagues' approach that does selective planning based on ``model inadequacy''~\cite{abbas2020selective}, which is analogous to using a stable model.

\paragraph{SynthDyna.}

\begin{wrapfigure}{R}{0.5\textwidth}
\begin{minipage}{0.5\textwidth}
\begin{algorithm}[H]
    \caption{$\Lcal$ : SynthDyna Meta Loss}
    \label{algo:metaloss}
    \begin{algorithmic}[1]
    \STATE \textbf{input:} $\thetabf$, $\phibf$, $r$,$\phibf'$
    \STATE $\thetabf'\gets\thetabf$
    \FOR{$1,\cdots, k$}
            \STATE $\Tilde{\phibf}, \Tilde{r}, \Tilde{\phibf}' \sim m(\etabf)$
            \STATE $\tilde{\delta} \gets \Tilde{r} + \gamma\thetabf'^{\top}\Tilde{\phibf}'  - \thetabf'^{\top}\Tilde{\phibf}$
        \STATE$\thetabf' \gets \thetabf' + \zeta \Tilde{\delta}\Tilde{\phibf}$
    \ENDFOR
    \STATE \textbf{return:} $( r + \gamma \thetabf^{\top}\phibf' - \thetabf'^{\top}\phibf)^2$
    \end{algorithmic}
\end{algorithm}
\begin{algorithm}[H]
    \caption{SynthDyna for Prediction}
    \label{algo:synthdyna}
    \begin{algorithmic}[1]
    \STATE \textbf{input:} 
    feature transform $f$
    \STATE \textbf{initialize:} $\thetabf_{0}$, $\etabf$, $\Dcal \gets \{\}$, 
    $\phibf_0\gets f(\hbf_0)$
    \FOR{ $t=1,2,\cdots$ }
        \STATE Observe $\obf_{t}$ and $r_{t}$ from the environment.\;
        \STATE $\phibf_{t}\gets f(\hbf_{t})$
        \STATE $\Dcal \gets \Dcal \cup \{\thetabf_{p}, \phibf_{t-1}, r_{t}, \phibf_{t} \}$\;
        \STATE // Update value with veridical experience.
        \STATE $\delta_t \gets r_t + \gamma \thetabf_{t-1}^\top\phibf_{t}  - \thetabf_{t-1}^\top\phibf_{t-1} $
        \STATE $\thetabf_{t} \gets \thetabf_{t-1} + \alpha \delta_t \phibf_{t-1}$
        \STATE $\thetabf_{p} \gets \thetabf_{t}$ // Save $\thetabf_t$ for model training
        \STATE // Update value with internal experience.
        \FOR{$1,\cdots, k$}
            \STATE $\Tilde{\phibf}, \Tilde{r}, \Tilde{\phibf}' \sim m(\etabf)$
            \STATE $\Tilde{\delta} \gets \Tilde{r} + \gamma\thetabf_{t}^{\top}\Tilde{\phibf}' - \thetabf_{t}^{\top}\Tilde{\phibf}$\;
            \STATE $\thetabf_t \gets \thetabf_t + \beta \Tilde{\delta}\Tilde{\phibf}$
        \ENDFOR
        \STATE // Update SynthDyna model with metaloss $\Lcal$ (Alg~\ref{algo:metaloss}).
        \STATE Sample minibatch $\Bcal$ from $\Dcal$. 
        \STATE $\etabf \gets$ Adam($\etabf, \Bcal, \Lcal$)
    \ENDFOR
    \end{algorithmic}
\end{algorithm}
\end{minipage}
\end{wrapfigure}


In Algorithm \ref{algo:synthdyna}, we propose a proof-of-concept algorithm that we call {\it SynthDyna}. SynthDyna's model, unlike our planning baselines, is (i) learned, and (ii) generates synthetic internal planning experience that does not aim to accurately simulate environment dynamics. In designing SynthDyna, we seek to demonstrate that giving the model control over the representation of a transition for planning while having a model learning objective informed by utility to the learning system, can result in faster adaptation in our reward switching non-stationary environment.  Like a typical Dyna architecture, SynthDyna updates its value function parameters using both veridical experience transitions and internal planning transitions from its model. Unlike our other Dyna representatives, SynthDyna’s model is a generative model, using a 2-layer fully connected neural network with parameters $\etabf$ that takes in a Gaussian noise vector as input and produces a transition tuple ($\tilde{\phibf}, \tilde{r}, \tilde{\phibf'} $). SynthDyna trains this generative model with a meta-learning procedure that uses a meta-loss tied to the learner's primary objective.
As seen in line 6 of Algorithm \ref{algo:synthdyna}, SynthDyna collects every experienced veridical transition ($\phibf_{t-1}, r_{t}, \phibf_{t}$) and it also collects the value function parameters, $\thetabf_{p}$  from the previous step before planning updates were performed. At a given frequency of steps in the environment, SynthDyna samples a batch of tuples where each is 
($\thetabf_{p}, \phibf, r, \phibf'$).
It then uses that batch’s samples as inputs to the metaloss; see Algorithm \ref{algo:metaloss}. The metaloss carries out the process of planning with SynthDyna’s model for $k$ steps updating $\thetabf_{p}$ from the batch. It then evaluates the squared TD-error, on the stored veridical transition ($\phibf, r, \phibf'$) from the batch.  To prevent the TD-target in the error from being manipulated before this evaluation step, the value function parameters in the target are set as the initial parameters $\thetabf$.  The gradient of the resulting meta-loss is then used to update the model parameters $\etabf$ to improve the model. 

We hypothesize that the combination of SynthDyna’s meta-learned model that aims to be useful rather than accurate, and giving the model full control over the representation of the internal planning transitions will yield a model that can surpass the learning speed of all baselines including our strongest baseline, StableExperienceDyna, that is unfairly informed of stable transitions in the environment. 

\paragraph{Results.}

\begin{figure*}
    \centering
  \centering
	\begin{subfigure}[b]{0.49\textwidth}
		\centering
		\includegraphics[width=\textwidth]{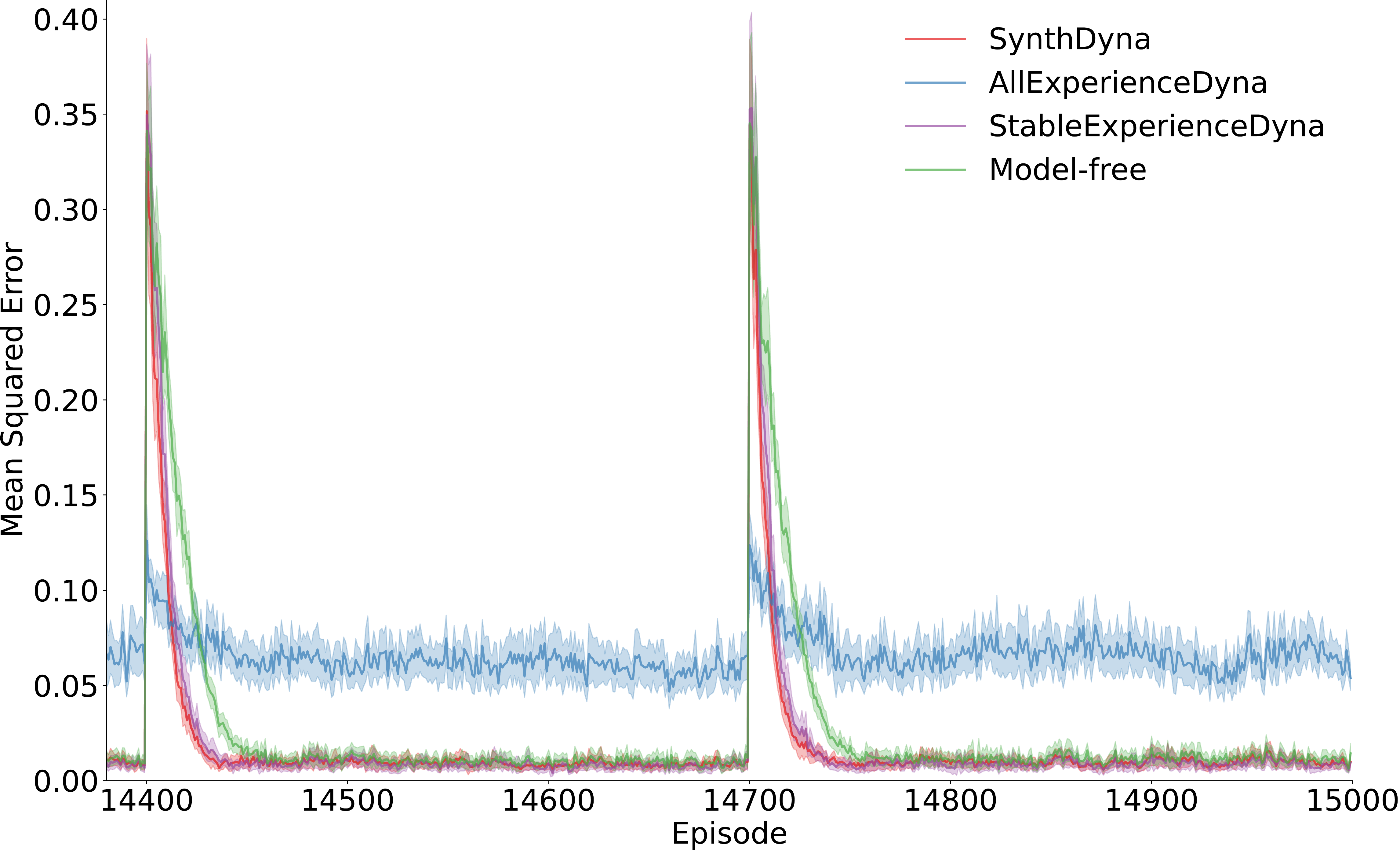}
		\caption{\footnotesize{Mean squared error per episode for each algorithm }}\label{fig:learning_curve}
	\end{subfigure}
        \hfill
	\begin{subfigure}[b]{0.49\textwidth}
		\centering
		\includegraphics[width=\textwidth]{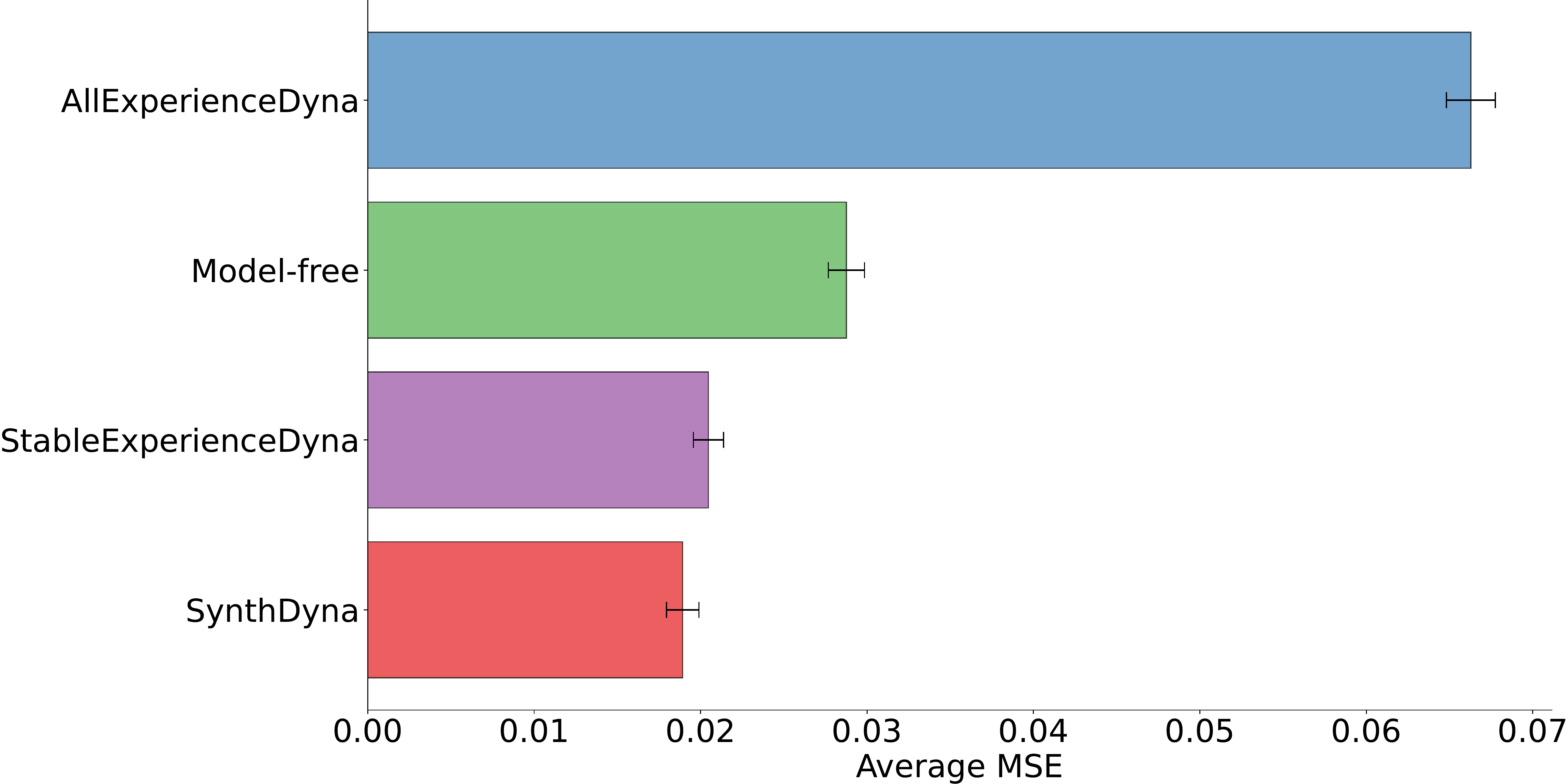}
		\caption{\footnotesize{Average MSE per algorithm}}\label{fig:bar_graph}
	\end{subfigure}
    \caption{We ran each algorithm in the non-stationary windy hallway environment for 15,000 episodes, with 30 trials. Fig. \ref{fig:learning_curve}, shows the mean squared value error per episode for the last 600 episodes.  Fig. \ref{fig:bar_graph}, shows the average MSE over the last 600 episodes. Shaded regions and error bars show 95\% confidence intervals.}
    \label{fig:results}
\end{figure*}

We evaluate SynthDyna against our three baselines in the the non-stationary hallway environment described previously. Every algorithm's hyper-parameters are selected after a comprehensive grid search. The primary metric we focus on is the mean squared value error (MSE) per episode, which measures the average squared error of the predicted return compared to the expected return over the timesteps in an episode. Let $\thetabf$ be the value function’s parameters before an update based on a veridical transition,  ($\phibf, r, \phibf'$), (i.e., line 9 of Algorithm \ref{algo:synthdyna}).  For an episode of length $n$, starting at timestep $e$, the mean squared value errors per episode is computed with: 
$\mbox{MSE} = \frac{1}{n}\sum_{t=e}^{e+n}{(v(\hbf_{t}) - \phibf_{t}^{\top}\thetabf_{t})^2}$.

Figure \ref{fig:learning_curve} shows the MSE per episode for the last two reward-regime switches. Figure \ref{fig:bar_graph} highlights the average MSEs for each algorithm for the last 600 episodes. AllExperienceDyna has the poorest performance across baselines due to its planning process that updates the value function with old experience that could be in direct opposition to the current reward regime. In an environment like this one, the much simpler Model-free algorithm performs substantially better as each update to the value function is exclusively done with the present moment’s veridical experience. StableExperienceDyna is able to surpass the performance of AllExperienceDyna and Model-free demonstrating the value of model-based RL for fast adaptation in this domain.  
It can reap the benefits of planning by updating based on past transition tuples that are stable across switching reward regimes. SynthDyna outperforms even StableExperienceDyna with a statistically significant difference between the average MSEs over the last 600 episodes (two sample t-test, $p < 0.05$). This shows that a model like SynthDyna’s can be built in a reward switching non-stationary environment to produce completely synthetic transition tuples and yet update the value function more effectively than if a stable and accurate model of the world was used.








\section{Conclusion}


``All models are wrong but some are useful''~\cite{box1979robustness}. In this work, we questioned whether models in MBRL should aim to accurately simulate environment dynamics.  We posited an alternative approach that models should aim to be useful to the learning system, aiding in faster adaptation.  As a proof of concept, we introduced SynthDyna, which uses meta-learning to train a model for Dyna-style planning such that it directly reduces prediction error.  In our experiments, SynthDyna outperforms model-based baselines focused on accuracy.  Focusing on model usefulness and the approach taken by SynthDyna are promising directions for MBRL to realize its potential for general, sample efficient RL.

\bibliography{ref}

\begin{thebibliography}{1}

\bibitem{abbas2020selective}
Z.~Abbas, S.~Sokota, E.~Talvitie, and M.~White.
\newblock Selective {D}yna-style planning under limited model capacity.
\newblock In {\em International Conference on Machine Learning}, pages 1--10.
  PMLR, 2020.

\bibitem{box1979robustness}
G.~E. Box.
\newblock Robustness in the strategy of scientific model building.
\newblock In {\em Robustness in statistics}, pages 201--236. Elsevier, 1979.

\bibitem{farahmand2017value}
A.-m. Farahmand, A.~Barreto, and D.~Nikovski.
\newblock Value-aware loss function for model-based reinforcement learning.
\newblock In {\em Artificial Intelligence and Statistics}, pages 1486--1494.
  PMLR, 2017.

\bibitem{grimm2020value}
C.~Grimm, A.~Barreto, S.~Singh, and D.~Silver.
\newblock The value equivalence principle for model-based reinforcement
  learning.
\newblock {\em Advances in Neural Information Processing Systems},
  33:5541--5552, 2020.

\bibitem{holland2018effect}
G.~Z. Holland, E.~J. Talvitie, and M.~Bowling.
\newblock The effect of planning shape on {D}yna-style planning in
  high-dimensional state spaces.
\newblock {\em arXiv preprint arXiv:1806.01825}, 2018.

\bibitem{nosuccessful_planning}
M.~C. Machado, M.~G. Bellemare, E.~Talvitie, J.~Veness, M.~Hausknecht, and
  M.~Bowling.
\newblock Revisiting the {A}rcade {L}earning {E}nvironment: Evaluation
  protocols and open problems for general agents.
\newblock {\em Journal of Artificial Intelligence Research}, 61:523--562, 2018.

\bibitem{schrittwieser2020mastering}
J.~Schrittwieser, I.~Antonoglou, T.~Hubert, K.~Simonyan, L.~Sifre, S.~Schmitt,
  A.~Guez, E.~Lockhart, D.~Hassabis, T.~Graepel, et~al.
\newblock Mastering {A}tari, {G}o, chess and shogi by planning with a learned
  model.
\newblock {\em {N}ature}, 588(7839):604--609, 2020.

\bibitem{silver2017predictron}
D.~Silver, H.~Hasselt, M.~Hessel, T.~Schaul, A.~Guez, T.~Harley,
  G.~Dulac-Arnold, D.~Reichert, N.~Rabinowitz, A.~Barreto, et~al.
\newblock The predictron: End-to-end learning and planning.
\newblock In {\em International conference on machine learning}, pages
  3191--3199. PMLR, 2017.

\bibitem{sutton1991dyna}
R.~S. Sutton.
\newblock Dyna, an integrated architecture for learning, planning, and
  reacting.
\newblock {\em ACM Sigart Bulletin}, 2(4):160--163, 1991.

\end{thebibliography}
\bibliographystyle{abbrv}
\end{document}